\documentclass[10pt,twocolumn,letterpaper]{article}

\usepackage[pagenumbers]{cvpr} 

\definecolor{cvprblue}{rgb}{0.21,0.49,0.74}
\usepackage[pagebackref,breaklinks,colorlinks,allcolors=cvprblue]{hyperref}
\usepackage{amsmath,amssymb,bm}
\usepackage{multirow}
\usepackage{xcolor}
\usepackage[table]{xcolor}
\usepackage{array}
\usepackage{multirow} 
\usepackage{booktabs}
\usepackage{amsmath}
\usepackage{amsthm}
\usepackage{xcolor}
\usepackage{amssymb}
\usepackage{amsmath,amssymb}   
\usepackage{algorithm}         
\usepackage{algpseudocode}     
\usepackage{xcolor}            
\usepackage{float}             

\usepackage{amsmath,amssymb}   
\usepackage{xcolor}            
\usepackage{algorithm}
\usepackage{algpseudocode}     
\usepackage{float}             
\usepackage{amsthm}  
\usepackage[misc]{ifsym}
\usepackage{microtype}

\newtheorem{theorem}{Theorem}

\DeclareMathOperator*{\argmin}{arg\,min}

\usepackage{amsthm}

\theoremstyle{definition}
\newtheorem{definition}{Definition}

\def\paperID{****} 
\def\confName{CVPR}
\def\confYear{2026}

\title{\textit{VideoCompressa}: Data-Efficient Video Understanding via Joint Temporal Compression and Spatial Reconstruction}

\author{
  {\bf
    Shaobo Wang$^{\spadesuit}$$^{* {\text{\Letter}}}$ \quad 
    Tianle Niu$^{\clubsuit \spadesuit}$$^*$ \quad
    Runkang Yang$^{\spadesuit}$  \quad
    Deshan Liu$^{\spadesuit}$ \quad
    Xu He$^{\spadesuit}$  
    \vspace{3pt}
  } \\
  {
  \bf
    Zichen Wen$^{\spadesuit}$ \quad 
    Conghui He$^{\diamondsuit}$ \quad 
    Xuming Hu$^{\heartsuit}$ \quad 
    Linfeng Zhang
    $^{\spadesuit}$$^{\text{\Letter}}$
    \vspace{2pt}
  } \\
    {
    $^{\spadesuit}$ EPIC Lab, SJTU $\quad$
    $^{\clubsuit}$ KTH $\quad$
    $^{\diamondsuit}$ Shanghai AI Laboratory $\quad$
    $^{\heartsuit}$ HKUST $\quad$
    \vspace{2pt}
  } \\
    {
  $^*$ Equal contribution $\quad$ 
    $^{\text{\Letter}}$ Corresponding authors
    }
}

\begin{document}
\maketitle
\begin{abstract}
The scalability of video understanding models is increasingly limited by the prohibitive storage and computational costs of large-scale video datasets. While data synthesis has improved data efficiency in the image domain, its extension to video remains challenging due to pervasive temporal redundancy and complex spatiotemporal dynamics. In this work, we uncover a critical insight: the primary source of inefficiency in video datasets is not inter-sample redundancy, but intra-sample frame-level redundancy. To leverage this insight, we introduce VideoCompressa, a novel framework for video data synthesis that reframes the problem as dynamic latent compression. Specifically, VideoCompressa jointly optimizes a differentiable keyframe selector—implemented as a lightweight ConvNet with Gumbel-Softmax sampling—to identify the most informative frames, and a pretrained, frozen Variational Autoencoder (VAE) to compress these frames into compact, semantically rich latent codes. These latent representations are then fed into a compression network, enabling end-to-end backpropagation. Crucially, the keyframe selector and synthetic latent codes are co-optimized to maximize retention of task-relevant information. Experiments show that our method achieves unprecedented data efficiency: on UCF101 with ConvNets, VideoCompressa surpasses full-data training by 2.34\% points using only 0.13\% of the original data, with over 5800× speedup compared to traditional synthesis method. Moreover, when fine-tuning Qwen2.5-7B-VL on HMDB51, VideoCompressa matches full-data performance using just 0.41\% of the training data—outperforming zero-shot baseline by 10.61\%.
\end{abstract}    
\section{Introduction}
\label{sec:intro}

\begin{figure}[tb!]
    \centering
    \includegraphics[width=.99\linewidth]{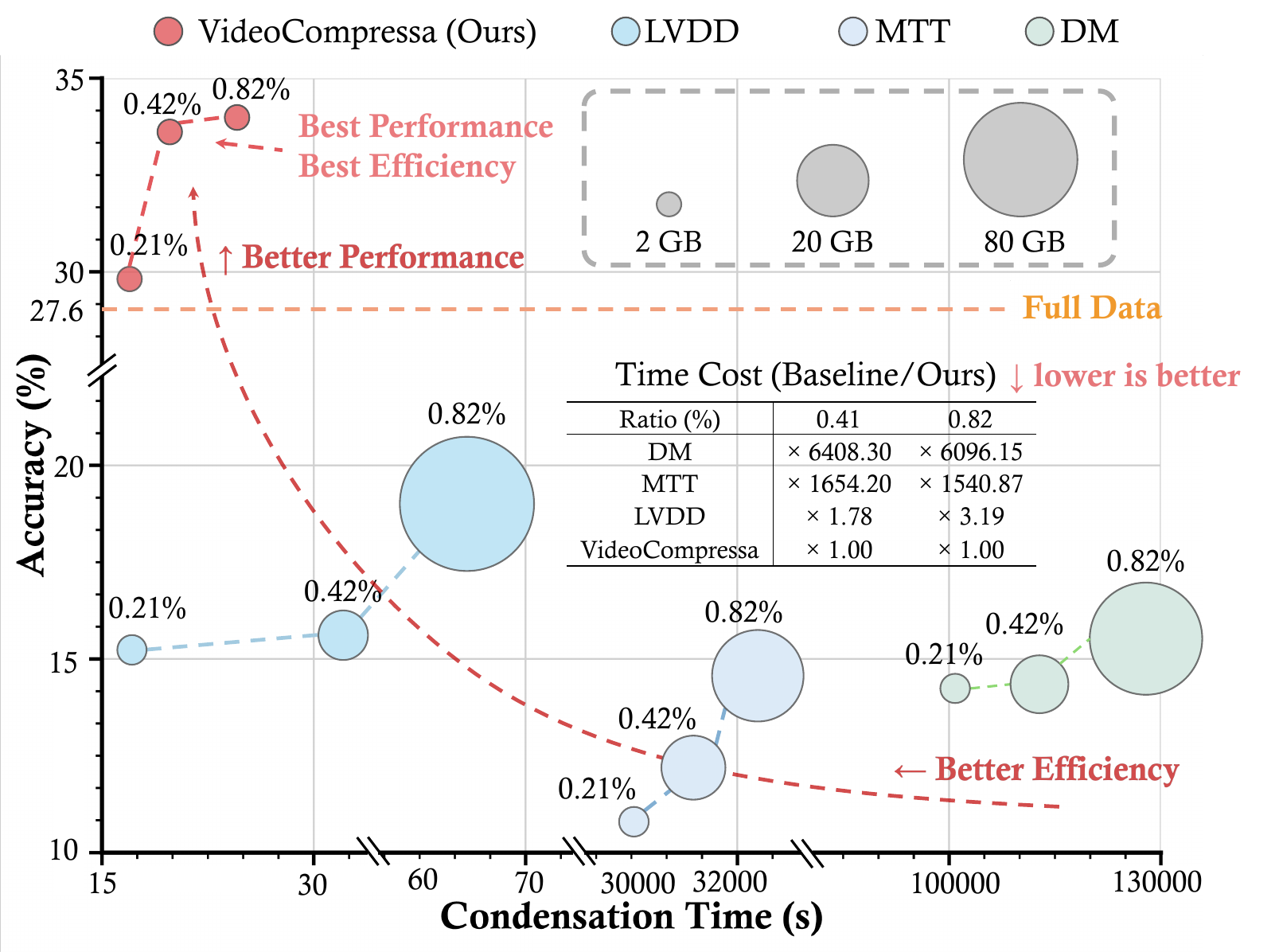} 
    \vspace{-5pt}
    \caption{
    Comparison of performance and computation efficiency across different compressing methods under multiple compression ratios on UCF101 dataset. Bubble size indicates peak GPU memory usage during compression. VideoCompressa achieves \textcolor{red!50}{\textbf{substantially better performance}} with dramatically lower computation cost, reaching \textcolor{blue!50}{\textbf{over 5800× speedup compared to prior compression methods}} such as DM, MTT, and LVDD. The dashed orange line marks full-data supervised training. 
    }
    \label{fig:efficiency}
    \vspace{-10pt}
\end{figure}

Recent advances in vision-language models (VLMs) have been largely driven by the availability of large-scale vision datasets~\citep{kay2017kinetics, arnab2021vivit, tang2025video, achiam2023gpt, wu2025qwen}. However, as the data grows exponentially, the training and storage costs associated with such datasets have become major bottlenecks~\citep{feichtenhofer2019slowfast, sorscher2022beyond, wang2024internvideo2, zhang2025survey}. This data-scaling challenge has motivated a growing field of research focused on data-efficient learning~\citep{solko2024data, sajedi2024data,kong2025multi}. A highly promising solution is dataset compression, which has achieved tremendous success in the image domain~\citep{DM, MTT, wang2018dataset, cazenavette2022dataset, lei2023comprehensive, sajedi2023datadam, sun2024diversity, wang2025dataset,min2025imagebinddc,xu2025rethinking,wang2025drupi,wang2024samples}. However, translating this success to the video domain is non-trivial and presents unique challenges~\citep{tong2022videomae, ravanbakhsh2024deep, shu2025video}. First, videos introduce a temporal dimension, which is rife with redundancy, adjacent frames are often nearly identical, naively compressing all frames is computationally infeasible and informationally wasteful~\citep{huang2018makes, lin2019tsm, ye2025re}. Second, the sheer scale of video data makes many pixel-space optimization methods, common in image distillation, intractable~\citep{wang2023dancing, zhao2024video, chen2024large}. Third, the complex spatiotemporal dynamics that define an action may be subtle and occur in short bursts, making simple heuristic-based frame selection highly suboptimal~\citep{kay2017kinetics, girdhar2019video, hutchinson2021video}. Despite the prevalence of keyframe selection~\citep{wolf1996key, mo2021keyframe, gan2023keyframe} as a technique for efficient video processing in large multimodal models~\citep{tang2025adaptive, hu2025m, yan2025enhanced}, its potential for the task of large-scale video action dataset compression still remains unexplored.

\begin{figure*}[tb!]
    \centering
    \includegraphics[width=0.99\linewidth]{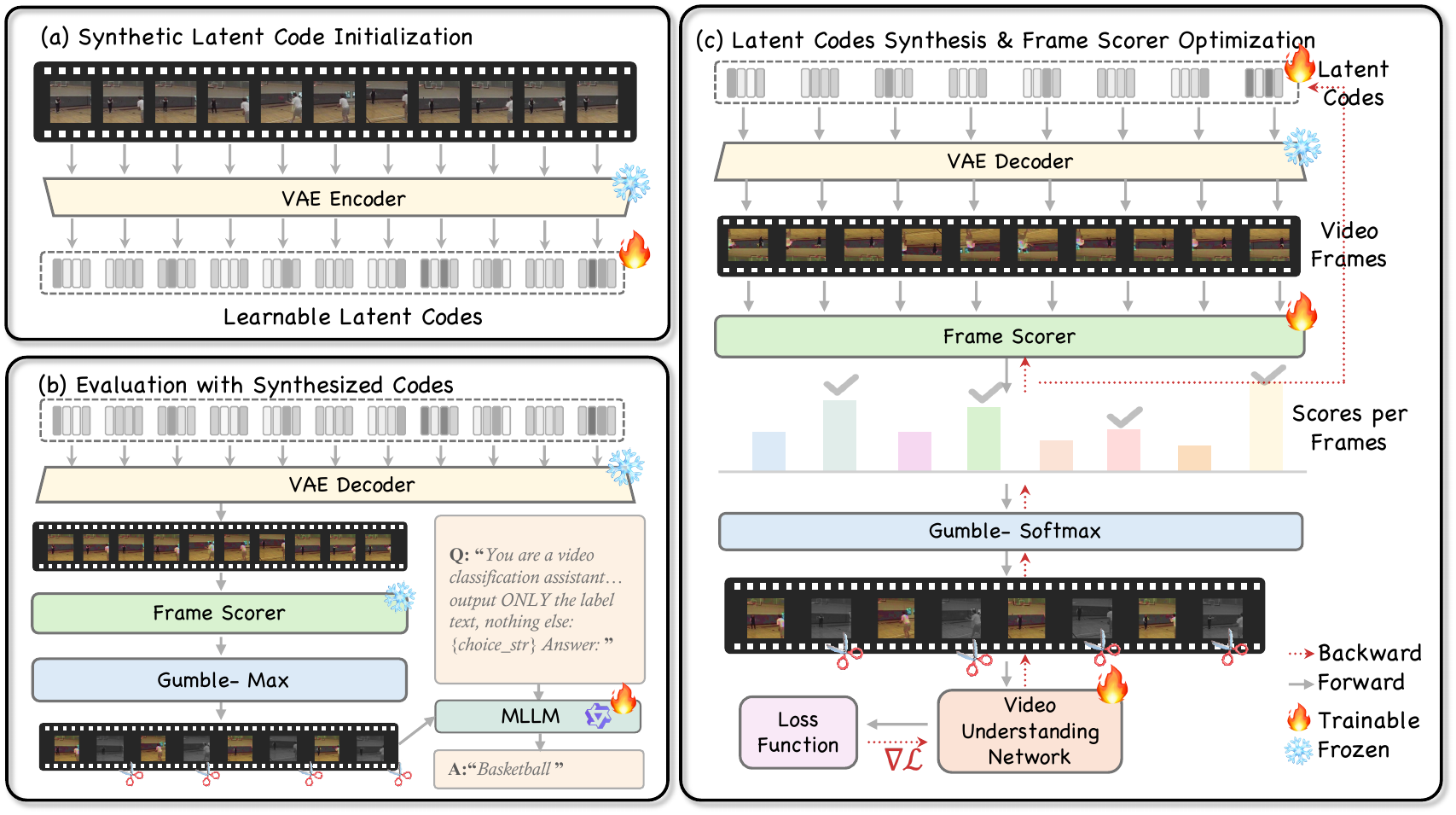}
    \caption{An overview of our proposed pipeline. A raw video is first processed by a Gumbel-based keyframe selection module. The selected keyframes are then encoded into the latent space by a frozen VAE. Finally, these latent representations are optimized via latent space synthesis, guided by the gradients from a student model, to form the synthetic dataset.}
    \label{fig:pipeline}
\end{figure*}

Prior attempts at video dataset compression also have several limitations. Coreset selection methods~\citep{guo2022deepcore, chai2023efficient, lee2024coreset, moser2025coreset} simply select a subset of real videos, which is storage-inefficient and fails to create a densified representation. Pixel-space compression methods like VDSD~\citep{wang2023dancing} and DATM~\citep{chen2024large} attempt to synthesize video pixels, but struggle with the high computational cost, limiting their scalability and compression ratios. Though IDTD~\citep{zhao2024video} jointly considers the unique within-sample and inter-sample redundancies of video, it still struggles to capture sufficient information diversity at low sample settings. More recently LVDD~\citep{li2025latent} explored latent space compression, but it typically decouple temporal selection from the latent optimization process, which limits its flexibility and generalization. 

To this end, we propose VideoCompressa, a framework that for the first time, jointly learns to select salient keyframes and synthetize a compact latent video dataset in an end-to-end manner. Our approach tackles temporal and spatial redundancy simultaneously. Instead of relying on fixed heuristics or decoupled pipelines~\citep{li2025latent}, we treat both the keyframe selection and the synthetic data content as learnable parameters of a single, unified compression objective.
Our method is built on three key ideas. First, we move the compression process from the high-dimensional pixel space to the compact latent space through a lightweight pre-trained VAE encoder~\citep{taesd}.
Second, instead of using a fixed temporal sampling heuristic, we introduce a differentiable keyframe selection module powered by the Gumbel-Softmax reparameterization~\citep {nadarajah2004beta, jang2016categorical, huijben2022review} to selects the most salient $K$ frames for visual understanding. Third, we jointly optimize the synthetic latent codes and the parameters of the keyframe selection network. 
Extensive experiments show our method achieves new state-of-the-art results, dramatically outperforming existing methods on large-scale benchmarks (as shown in Figure~\ref{fig:efficiency}) and achieving lossless compression on UCF101 dataset. Critically, we demonstrate that our highly compressed datasets can also be used to fine-tune Multimodal Large Language Model (MLLM), achieving performance remarkably close to models fine-tuned on the entire dataset. Contributions are summarized as follows:
\begin{enumerate}
    \item We identify  intra-sample frame-level redundancy, rather than inter-sample redundancy, as the principal source of inefficiency in large-scale video datasets. We argue that effective data compression must therefore prioritize principled temporal information selection over uniform sequence compression.

    \item To act on this insight, we introduce \textbf{\textit{VideoCompressa}}, a novel end-to-end framework for video data compression. Our method uniquely co-optimizes a differentiable Gumbel-Softmax keyframe selector with the reconstruction of compact latent codes, ensuring that both \textit{which} frames to select and \textit{how} to represent their content are jointly guided by the task objective.

    \item We demonstrate state-of-the-art performance and unprecedented data efficiency, achieving a state of \textbf{lossless data compression} on several datasets. For standard ConvNets on UCF101, our method surpasses full-data performance by 2.34\% using merely \textbf{0.13\%} of the original data. This remarkable efficiency generalizes to modern architectures; when fine-tuning a MLLM on HMDB51, our method matches full-data performance using just \textbf{0.41\%} of the data.
\end{enumerate}

\section{Related Work}
\label{sec:relatedwork}

\noindent \textbf{Key-Frame Selection for Video Understanding.}
Key-frame selection is a cornerstone of efficient video understanding~\citep{wolf1996key, asha2018key, mo2021keyframe, gan2023keyframe, yan2025enhanced}. The field has progressed from fast but rigid heuristics (e.g., uniform sampling, pixel differences)~\citep{asha2018key, zhang2023video, lin2024video} to more sophisticated, differentiable schemes that learn to select frames based on task relevance and temporal coverage~\citep{tang2025adaptive}. However, the primary application of these advanced selection methods has been to accelerate inference on pre-trained models~\citep{zhang2023video, lin2024video, bai2024survey}. Their potential for creating a synthetic dataset from scratch—as a core component of the data generation process itself—remains largely unexplored. Our work bridges this gap by directly integrating a differentiable key-frame selection mechanism into the dataset synthesis pipeline. This novel approach enables, for the first time, the joint optimization of \textit{which} frames to select and \textit{how} to represent their content, tackling both temporal and spatial redundancy in a unified framework.

\noindent\textbf{Efficient Dataset Synthesis.}
The quest for data efficiency has spurred the development of powerful dataset compression techniques, which create compact data surrogates that retain the performance of the full dataset, primarily in the image or language domain~\citep{wang2018dataset, lei2023comprehensive, yu2023dataset, sajedi2023datadam, sun2024diversity, wang2025dataset}. Extending this paradigm to video is complicated by immense temporal redundancy~\citep{lang1995defining, huang2018makes, tong2022videomae, shu2025video, liu2025video,liu2025shifting}. Early video-centric approaches like VDSD~\citep{wang2023dancing} and DATM~\citep{chen2024large} attempt to synthesize video pixels directly. While VDSD disentangles static and dynamic content, both methods grapple with the high computational cost and scalability issues inherent to pixel-space optimization, without explicitly learning which frames are most informative. To mitigate this, recent work like LVDD~\citep{li2025latent} has moved the synthesis process into a more efficient latent space. However, it relies on fixed, selection-centric heuristics rather than a learnable optimization process, limiting its flexibility and the quality of the resulting summary. Consequently, a critical gap remains: current methods condense video sequences but fall short of strategically selecting the most salient temporal information, leaving a rich source of efficiency untapped.

\section{Preliminary: Data-Efficient Video Understanding}
\label{sec:preliminary}
Given a video training dataset $\mathcal{D}_{\text{train}} = \{(\mathcal{V}, y)\}_{i=1}^{N}$, where $N = |\mathcal{D}_{\text{train}}|$ and each $\mathcal{V} = \{f_{i,t}\}_{t=1}^{T}$ is a sequence of $T$ frames, and a held-out test set $\mathcal{D}_{\text{test}}$ sampled from the same distribution, the goal of \textit{data-efficient video understanding} is to construct a compact, compressed representation $\mathcal{S}$ of $\mathcal{D}_{\text{train}}$ that preserves its discriminative power for video understanding tasks. This is achieved by designing an algorithm $\mathcal{A}$ that transforms $\mathcal{D}_{\text{train}}$ into a compact dataset $\mathcal{S} = \mathcal{A}(\mathcal{D}_{\text{train}})$ with size $M = |\mathcal{S}| \ll N$. Mathematically, we formalize this objective as follows:

\begin{definition}[Data-Efficient Video Understanding]
\label{def:devu}
    Let $\mathcal{M}(\cdot)$ denote a learning algorithm that maps a dataset to a model, and let $\mathfrak{M}$ be a class of admissible learning algorithms. Given a fixed compressed dataset size $M \ll N$, the goal is to design an algorithm $\mathcal{A}$ that produces $\mathcal{S} = \mathcal{A}(\mathcal{D}_{\text{train}})$ with $|\mathcal{S}| = M$, such that the worst-case performance gap on the test set is bounded:
    \begin{equation}
        \sup_{\mathcal{M} \in \mathfrak{M}} \left| \mathcal{L}(\mathcal{M}(\mathcal{S}); \mathcal{D}_{\text{test}}) - \mathcal{L}(\mathcal{M}(\mathcal{D}_{\text{train}}); \mathcal{D}_{\text{test}}) \right| \leq \delta(M),
    \end{equation}
    where $\mathcal{L}(f; \mathcal{D}_{\text{test}}) = \mathbb{E}_{(\mathcal{V},y)\sim\mathcal{D}_{\text{test}}}[\ell(f(\mathcal{V}), y)]$ is the expected test loss and $\delta(M) \geq 0$ is a tolerance that depends on $M$. The objective is to minimize $\delta(M)$ for a given $M$.
\end{definition}
\section{Methodology}
\label{sec:Methodology}

We propose \textit{VideoCompressa}, a unified framework for data-efficient video understanding that jointly optimizes temporal compression and spatial reconstruction. As illustrated in Figure~\ref{fig:pipeline}, our approach processes each input video $\mathcal{V}$ through two co-optimized stages: (1) \textit{Gumbel-based temporal compression} that dynamically selects a minimal set of informative frames, and (2) \textit{latent spatial reconstruction} that refines visual content in the VAE latent space. Crucially, both stages are differentiable and jointly trained with the downstream video understanding network, enabling end-to-end optimization of the compressed representation $\mathcal{S}$ under the data-efficiency constraint $|\mathcal{S}| = M$. The entire pipeline is trained to minimize the performance gap $\delta(M)$ defined in Section~\ref{sec:preliminary}.

\noindent \textbf{Latent Code Initialization.} Before optimization, each video $\mathcal{V} = \{f_{t}\}_{t=1}^{T}$ is encoded into an initial latent representation using a pre-trained VAE with an encoder $\mathcal{E}$ and a decoder $\psi$. Specifically, we apply the VAE encoder $\mathcal{E}$ to each frame to obtain latent codes $\mathbf{z}_{t} = \mathcal{E}(f_{t}) \in \mathbb{R}^d$, forming the initial latent sequence $\mathbf{Z} = \{\mathbf{z}_{t}\}_{t=1}^{T}$. These latent codes serve as (i) the basis for decoding pixel frames via $\psi$ that feed into temporal compression, and (ii) optimization variables for spatial reconstruction.

\subsection{Gumbel-Based Temporal Compression}
\label{subsec:temporal}

We first perform temporal compression to select key-frames within each video. This is achived by training a \textit{frame scorer} via Gumbel-Softmax technique. The frame selector operates on decoded pixel frames to capture visual semantics. For each input latent code $\mathbf{Z} = \{\mathbf{z}_{t}\}_{t=1}^{T}$, we first obtain the decoded frame sequence $\hat{\mathcal{V}} = \{\hat{f}_{t}\}_{t=1}^{T}$, where $\hat{f}_{t} = \psi(\mathbf{z}_{t})$ and $\psi$ is the frozen VAE decoder.

Specifically, a lightweight neural network $\phi$, composed of two 2D convolutional layers followed by a linear layer, is used as the frame scorer. It processes the entire decoded sequence to produce a score vector $\mathbf{q} \in \mathbb{R}^{T}$, \emph{i.e.}, $\mathbf{q} = \phi\big( \hat{f}_{1}, \dots, \hat{f}_{T} \big)$,
where each element $q_{t}$ represents the importance score of frame $t$. These logits are normalized via softmax to form selection probabilities:
\begin{equation}
    p_{t} = \frac{\exp(q_{t})}{\sum_{t'=1}^{T} \exp(q_{t'})}.
\end{equation}
To enable differentiable selection of $K$ frames ($K \ll T$), we apply the Gumbel-Softmax reparameterization \cite{jang2016categorical} with temperature $\tau = 1$. For each of $K$ independent samples $k = 1, \dots, K$, we draw Gumbel noise per frame:
\begin{equation}
    g_{t}^{(k)} = -\log(-\log(u_{t}^{(k)})), \quad u_{t}^{(k)} \sim \text{Uniform}(0,1),
\end{equation}
and compute a soft selection distribution:
\begin{equation}
    \tilde{p}_{t}^{(k)} = \frac{\exp\left( (q_{t} + g_{t}^{(k)}) / \tau \right)}{\sum_{t'=1}^{T} \exp\left( (q_{t'} + g_{t'}^{(k)}) / \tau \right)}.
\end{equation}
Each sample yields one expected frame via soft aggregation:
\begin{equation}
    \hat{f}^{(k)} = \sum_{t=1}^{T} \tilde{p}_{t}^{(k)} \cdot \hat{f}_{t}.
\end{equation}
The final compressed video is formed by concatenating the $K$ selected frames:
\begin{equation}
    \tilde{\mathcal{V}} = \left[ \hat{f}^{(1)}, \hat{f}^{(2)}, \dots, \hat{f}^{(K)} \right].
\end{equation}

\noindent \textbf{Gradient Propagation via Gumbel-Softmax.} Gradients from the classification loss flow through each $\hat{f}^{(k)}$ via the Gumbel-Softmax path. The gradient w.r.t. the logits $\mathbf{q}$ is approximated by averaging over $K$ samples:
\begin{equation}
    \nabla_{\mathbf{q}} \mathcal{L} \approx \frac{1}{K} \sum_{k=1}^K \nabla_{\tilde{\mathbf{p}}^{(k)}} \mathcal{L} \cdot \nabla_{\mathbf{q}} \tilde{\mathbf{p}}^{(k)},
\end{equation}
where $\nabla_{\mathbf{q}} \tilde{\mathbf{p}}^{(k)}$ follows from the softmax Jacobian. Since $\phi$ operates on decoded frames and $\psi$ is frozen, gradients update only the parameters of $\phi$.

\noindent \textbf{Downstream Inference.} At inference, we perform deterministic top-$K$ selection: we select the $K$ frames with highest logits $\mathbf{q}$ to form $\tilde{\mathcal{V}} = \{ \hat{f}_{t} \}_{t \in \mathcal{T}_K}$, where $\mathcal{T}_K = \text{top-}K(\mathbf{q})$. This yields a compact representation with exactly $K$ frames, ensuring minimal computational overhead.

\subsection{Latent Spatial Reconstruction}
\label{subsec:spatial}

Following temporal compression, we optimize the latent codes of the selected frames rather than pixel-level representations to minimize storage overhead. Storing optimized latents requires significantly less space than full-resolution frames while preserving discriminative features. These latent codes are jointly updated with the video understanding model through end-to-end classification training. We optimize three components simultaneously: the latent codes $\mathbf{Z}$, the temporal scorer $\phi$, and the video understanding model $\mathcal{M}_\theta$. The objective is defined as:
\begin{equation}
    \argmin_{\mathbf{Z}, \phi, \theta} \sum_{i=1}^{N} \mathcal{L}_{\text{CE}} \left( \mathcal{M}_\theta \left( \tilde{\mathcal{V}}(\mathbf{Z}, \phi) \right), y \right),
\end{equation}
where $\tilde{\mathcal{V}}$ is the temporally compressed video constructed as in Section~\ref{subsec:temporal}. This single-objective formulation creates a closed optimization loop:
\begin{enumerate}[leftmargin=*,nosep]
    \item $\mathcal{M}_\theta$ learns features from compressed videos under supervision of original hard labels $y$
    \item $\phi$ adapts frame selection to maximize classification accuracy (via gradients through $\tilde{\mathcal{V}}$)
    \item $\mathbf{Z}$ are refined through gradients flowing backward through $\psi$ and $\mathcal{M}_\theta$
\end{enumerate}
Although the temporal scorer operates on decoded frames, gradients from the classification loss $\mathcal{L}_{\text{CE}}$ propagate through $\tilde{\mathcal{V}}$ to the latent codes $\mathbf{Z}$ via the frozen decoder $\psi$. Specifically, for each latent code $\mathbf{z}_{t}$, the gradient $\partial \mathcal{L}_{\text{CE}} / \partial \mathbf{z}_{t}$ adjusts its representation to improve classification performance. Crucially, no auxiliary losses or regularization terms are introduced—all optimization is driven solely by the classification signal. This ensures spatial reconstruction is strictly task-oriented, while the fixed frame budget $k$ inherently enforces temporal compression. The co-optimization directly minimizes the performance gap $\delta(M)$ from Definition~\ref{def:devu} through end-to-end training.
\section{Experiments}
\label{sec:Experiments}

\begin{table*}[tb!]
\vspace{-5pt}
\centering
\caption{Classification accuracy (\%) for data condensation on different video dataset synthesis benchmarks under different ratio settings. All methods use a randomly initialized ConvNet and FrameScore network to guide synthesis, with accuracy measured by training that same ConvNet from scratch on the resulting condensed data.}
\label{tab:main_results}
\vspace{-5pt}
\resizebox{0.99\textwidth}{!}{
\begin{tabular}{@{}cc|cccc|ccccccc|c@{}}
\toprule 
\multirow{2}{*}{Dataset}
 & \multirow{2}{*}{Ratio (\%)} & \multicolumn{4}{c|}{Coreset Selection} & \multicolumn{7}{c|}{Data Synthesis} & \multirow{2}{*}{Full Data} \\ 
 &  & Random & Herding & Forgetting & GraNd & DM+VDSD & MTT+VDSD & MTT & DM & DATM & LVDD & \textbf{VideoCompressa} & \\ 
\midrule
\multirow{3}{*}{HMDB51} 
& 0.21 & 13.49{\scriptsize ±0.08} & 17.26{\scriptsize ±0.51} & 18.82{\scriptsize ±0.24} & 17.30{\scriptsize ±0.48} & 7.20{\scriptsize ±0.32} & 6.94{\scriptsize ±0.11} & 5.99{\scriptsize ±0.33} & 6.29{\scriptsize ±0.10} & 10.01{\scriptsize ±0.36} & 7.78{\scriptsize ±0.29} & \cellcolor{green!10}\textbf{27.78{\scriptsize ±1.14}} & \multirow{3}{*}{\textcolor{black}{28.08{\scriptsize ±0.89}}} \\
& 0.41 & 13.56{\scriptsize ±0.32} & 16.16{\scriptsize ±0.99} & 19.60{\scriptsize ±0.24} & 16.58{\scriptsize ±0.39} & 6.99{\scriptsize ±0.10} & 7.03{\scriptsize ±0.29} & 6.11{\scriptsize ±0.42} & 6.32{\scriptsize ±0.22} & 10.02{\scriptsize ±0.41} & 11.60{\scriptsize ±0.30} & \cellcolor{green!10}\textbf{28.00{\scriptsize ±0.45}} &  \\
& 0.82 & 13.69{\scriptsize ±0.37} & 15.84{\scriptsize ±0.69} & 18.93{\scriptsize ±0.48} & 16.71{\scriptsize ±0.48} & 8.13{\scriptsize ±0.23} & 8.77{\scriptsize ±0.34} & 9.34{\scriptsize ±0.31} & 8.97{\scriptsize ±0.11} & 13.77{\scriptsize ±0.33} & 15.35{\scriptsize ±0.26} & \cellcolor{green!10}\textbf{28.51{\scriptsize ±0.07}} &  \\ 
\midrule
\multirow{3}{*}{UCF101} 
& 0.13 & 17.76{\scriptsize ±1.15} & 13.80{\scriptsize ±0.33} & 13.65{\scriptsize ±0.04} & 13.45{\scriptsize ±0.11} &14.72{\scriptsize$\pm$0.46}
&13.99{\scriptsize$\pm$0.28}
&10.76{\scriptsize$\pm$0.32}
&9.58{\scriptsize$\pm$0.09}
&11.89{\scriptsize$\pm$0.63}
&15.32{\scriptsize$\pm$0.61} & \cellcolor{green!10}\textbf{29.84{\scriptsize ±0.75}} & \multirow{3}{*}{\textcolor{black}{27.60{\scriptsize ±0.62}}} \\
& 0.26 & 17.62{\scriptsize ±0.76} & 13.59{\scriptsize ±0.34} & 14.02{\scriptsize ±0.38} & 13.71{\scriptsize ±0.11} & 13.98{\scriptsize ±0.39} & 14.44{\scriptsize ±0.53} & 11.11{\scriptsize ±0.10} & 9.61{\scriptsize ±0.31} & 12.14{\scriptsize ±0.62} & 16.47{\scriptsize ±0.44} & \cellcolor{green!10}\textbf{34.28{\scriptsize ±0.61}} &  \\
& 0.52 & 18.31{\scriptsize ±0.82} & 13.46{\scriptsize ±0.26} & 13.61{\scriptsize ±0.10} & 14.04{\scriptsize ±0.45} & 14.78{\scriptsize ±0.43} & 18.84{\scriptsize ±0.54} & 14.56{\scriptsize ±0.19} & 12.11{\scriptsize ±0.24} & 14.33{\scriptsize ±0.76} & 19.99{\scriptsize ±0.42} & \cellcolor{green!10}\textbf{34.03{\scriptsize ±0.27}} &  \\ 
\midrule
\multirow{3}{*}{SSv2} 
& 0.72 & 1.26{\scriptsize ±0.10} & 1.12{\scriptsize ±0.16} & 1.95{\scriptsize ±0.21} & 2.12{\scriptsize ±0.32} &4.04{\scriptsize$\pm$0.22}
&4.11{\scriptsize$\pm$0.63}
&5.01{\scriptsize$\pm$0.59}
&3.69{\scriptsize$\pm$0.26}
&5.10{\scriptsize$\pm$0.22}
&4.31{\scriptsize$\pm$0.77}& \cellcolor{green!10}\textbf{8.18{\scriptsize ±0.33}} & \multirow{3}{*}{\textcolor{black}{24.07{\scriptsize ±0.50}}} \\
& 1.45 & 2.07{\scriptsize ±0.11} & 3.18{\scriptsize ±0.33} & 4.17{\scriptsize ±0.32} & 3.94{\scriptsize ±0.51} & 4.10{\scriptsize ±0.35} & 5.99{\scriptsize ±0.50} & 5.12{\scriptsize ±0.46} & 4.01{\scriptsize ±0.11} & 5.02{\scriptsize ±0.48} & 5.99{\scriptsize ±0.22} & \cellcolor{green!10}\textbf{18.82{\scriptsize ±0.60}} &  \\
& 2.90 & 3.39{\scriptsize ±0.32} & 4.31{\scriptsize ±0.12} & 5.17{\scriptsize ±0.40} & 4.63{\scriptsize ±0.33} & 5.01{\scriptsize ±0.20} & 7.92{\scriptsize ±0.29} & 6.68{\scriptsize ±0.33} & 4.42{\scriptsize ±0.49} & 5.97{\scriptsize ±0.69} & 8.03{\scriptsize ±0.45} & \cellcolor{green!10}\textbf{23.97{\scriptsize ±0.71}} &  \\ 
\midrule
\multirow{3}{*}{Kinetics-400} 
& 0.04 & 1.14{\scriptsize ±0.13} & 1.23{\scriptsize ±0.11} & 1.79{\scriptsize ±0.50} & 2.01{\scriptsize ±0.39} & 4.82{\scriptsize ±0.05} & 4.71{\scriptsize ±0.09} & 3.44{\scriptsize ±0.02} & 3.04{\scriptsize ±0.33 }& 4.78{\scriptsize ± 0.33} & 4.48{\scriptsize ±0.75} & \cellcolor{green!10}\textbf{9.96{\scriptsize ±0.59}} & \multirow{3}{*}{\textcolor{black}{29.77{\scriptsize ±0.54}}} \\
& 0.09 & 3.33{\scriptsize ±0.20} & 3.07{\scriptsize ±0.44} & 4.24{\scriptsize ±0.48} & 4.38{\scriptsize ±0.44} & 5.45{\scriptsize ±0.33} & 4.44{\scriptsize ±0.23} & 3.42{\scriptsize ±0.11} & 3.03{\scriptsize ±0.57} & 5.32{\scriptsize ±0.31} & 6.67{\scriptsize ±0.21} & \cellcolor{green!10}\textbf{16.77{\scriptsize ±0.71}} &  \\
& 0.17 & 3.78{\scriptsize ±0.21} & 4.14{\scriptsize ±0.32} & 4.95{\scriptsize ±0.63} & 4.71{\scriptsize ±0.57} & 7.02{\scriptsize ±0.16} & 8.40{\scriptsize ±0.53} & 7.61{\scriptsize ±0.38} & 7.79{\scriptsize ±0.82} & 7.96{\scriptsize ±0.46} & 10.99{\scriptsize ±0.13} & \cellcolor{green!10}\textbf{19.97{\scriptsize ±0.66}} &  \\ 
\bottomrule
\end{tabular}
}
\vspace{-5pt}
\end{table*}

\subsection{Experimental Setup}

\noindent\textbf{Datasets.} We adopted medium-scale video datasets UCF101~\cite{ucf101} and HMDB51~\cite{hmdb51}, as well as large-scale video datasets Kinetics-400~\cite{k400} and Something-Something-V2 (SSv2)~\cite{ssv2} in our experiments. UCF101~\cite{ucf101} consists of 13,320 video clips spanning 101 action categories, while HMDB51~\cite{hmdb51} includes 6849 video clips from 51 action categories. Kinetics-400~\cite{k400} comprises a diverse collection of videos covering 400 human action classes, whereas SSv2 focuses on 174 motion-centric classes with subtle temporal dynamics. We followed previous works to report the top-1 classification accuracy for UCF101 and HMDB51, and the top-5 classification accuracy for Kinetics-400 and SSv2.

\noindent\textbf{Baselines.} We compared our method against two categories of baselines: (i) \textit{selection-based} compression and (ii) \textit{synthesis-based} compression. For selection-based compression, which selects a representative subset of real videos, we included four algorithms: Random, Herding~\citep{herding}, Forgetting~\citep{forgetting}, and GraNd~\citep{grand}. For synthesis-based method, which synthesizes a new, compact dataset, we benchmarked against several state-of-the-art approaches: DM~\citep{DM}, MTT~\citep{MTT}, DATM~\citep{datm}, and LVDD~\citep{lvdd}. To ensure a comprehensive evaluation against the latest advancements, we also integrated the static-dynamic disentanglement strategy from the recent VDSD~\citep{vdsd} with established methods. We also reported results for these hybrid baselines, DM+VDSD and MTT+VDSD, to thoroughly assess the efficacy of different compression strategies.

\noindent\textbf{Models.}
We employed a lightweight 2-layer 2D-CNN with a single linear layer as the frame scorer $f_{\text{scorer}}$, and a 3-layer 2D-CNN for video understanding. The VAE model is implemented using the pre-trained TAESD~\citep{taesd} model, a compact 2D VAE model that occupies only 9MB. For convolutional network experiments, we used the same 2D-CNN for evaluation. For MLLM experiments, we utilized the Qwen2.5-VL-7B~\citep{Qwen2.5-VL} model for evaluation.

\noindent\textbf{Implementation Details.}
All experiments were conducted on 8 NVIDIA A100 GPUs. For the latent code reconstruction, we used the Adam optimizer to update the synthetic latent codes $\mathbf{Z}$ and the scorer $\phi$. The learning rate was set to $1 \times 10^{-2}$ for the latent code and $1 \times 10^{-3}$ for the scorer. To ensure fairness, all experiments were conducted for 3 times. More details are included in Appendix.

\subsection{Results on Convolutional Networks}

\noindent\textbf{Comparison with Baselines.} 
As shown in Table~\ref{tab:main_results},  VideoCompressa approach significantly outperforms all competing methods across all datasets and settings. Specifically,  VideoCompressa surpasses all other methods in classification accuracy on the HMDB51, UCF101, Kinetics-400, and SSv2 datasets for all data compression ratios. A particularly compelling result is that \textit{our VideoCompressa achieves better performance with training on the full data of HMDB51 and UCF101, requiring only 0.82\% and 0.13\% of the respective entire datasets}.

\begin{table}[tb!]
\centering
\caption{Cross-architecture accuracy (\%) on the UCF101 dataset with a 0.13\% data budget. A random initialized ConvNet model is used to guide the iteration of synthetic datasets, which are then evaluated by training different architectures from scratch.}
\vspace{-5pt}
\label{tab:cross_architecture}
\resizebox{0.99\linewidth}{!}{
\begin{tabular}{@{}c|cccc@{}}
\toprule
\multirow{2}{*}{Method} & \multicolumn{4}{c}{Classification Network} \\ 
 & ConvNet & MLP & RNN & GRU \\ 
\midrule
Random & 16.66{\scriptsize $\pm$0.34} & 1.07{\scriptsize $\pm$0.16} & 1.15{\scriptsize $\pm$0.10} & 1.02{\scriptsize $\pm$0.29} \\
Forgetting & 7.31{\scriptsize $\pm$0.30} & 1.25{\scriptsize $\pm$0.12} & 1.24{\scriptsize $\pm$0.16} & 0.85{\scriptsize $\pm$0.17} \\
Herding & 11.45{\scriptsize $\pm$0.35} & 1.04{\scriptsize $\pm$0.08} & 1.05{\scriptsize $\pm$0.12} & 1.31{\scriptsize $\pm$0.21} \\
LVDD & 15.32{\scriptsize $\pm$0.61} & 15.45{\scriptsize $\pm$0.15} & 13.92{\scriptsize $\pm$0.52} & 14.44{\scriptsize $\pm$0.72} \\
\cellcolor{green!10}\textbf{VideoCompressa} 
& \cellcolor{green!10}\textbf{29.84{\scriptsize $\pm$0.75}}
& \cellcolor{green!10}\textbf{23.97{\scriptsize $\pm$0.53}}
& \cellcolor{green!10}\textbf{26.07{\scriptsize $\pm$0.57}}
& \cellcolor{green!10}\textbf{26.15{\scriptsize $\pm$0.32}} 
\\ 
\midrule
Full Data & 27.60{\scriptsize $\pm$0.62} & 24.87{\scriptsize $\pm$0.46} & 25.15{\scriptsize $\pm$0.33} & 26.99{\scriptsize $\pm$0.71} \\
\bottomrule
\end{tabular}
}
\vspace{-5pt}
\end{table}
\noindent\textbf{Cross-Architecture Performance.} We conducted a rigorous cross-architecture generalization experiment, with results presented in Table~\ref{tab:cross_architecture}. Following a standard protocol, we used a randomly initialized ConvNet to guide the reconstruction process. Subsequently, we used the resulting synthetic latent code to generate decoded synthetic dataset. The dataset is used to train a variety of architectures from scratch: a ConvNet, an MLP, an RNN, and a GRU. The results unequivocally demonstrate the superior generalization of the dataset synthesized by VideoCompressa. Our method consistently outperforms all competing approaches across every target architecture, highlighting that synthetic data is not overfitted to the guiding ConvNet but has instead captured more fundamental and transferable data characteristics. The performance gap is substantial; when compared to the strongest baseline for each model family, our approach achieves absolute accuracy gains of +13.18\% on ConvNet, +8.52\% on MLP, +12.15\% on RNN, and +11.71\% on GRU. 

\begin{figure}[tb!]
    \centering
   \includegraphics[width=.99\linewidth]{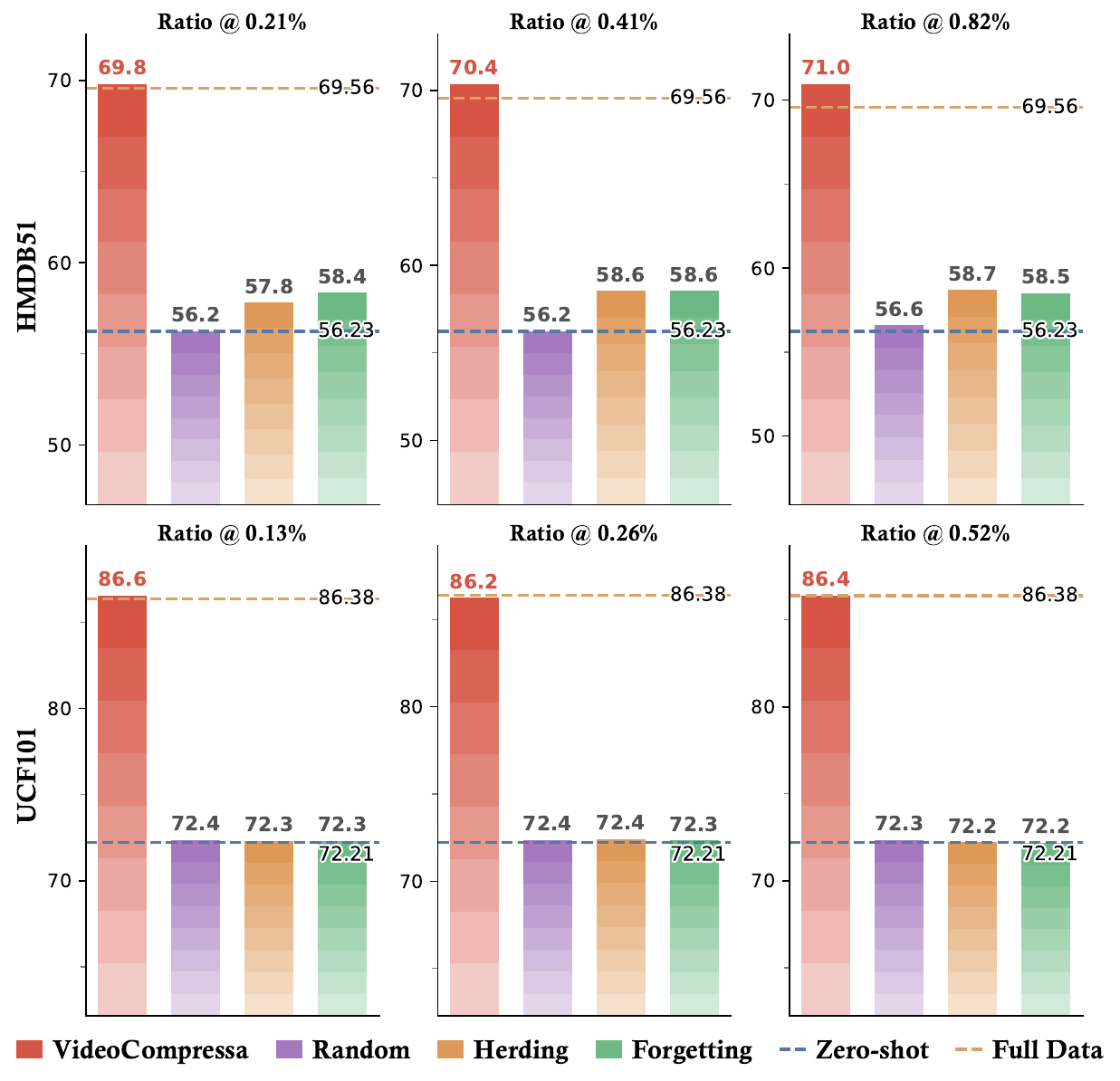} 
    \vspace{-5pt}
    \caption{
       Classification accuracy (\%) for data compression on HMDB51 and UCF101 video benchmarks under different ratio settings. All methods use a randomly initialized ConvNet and frame scorer to guide reconstruction, with accuracy measured by finetuning pre-trained Qwen2.5-VL-7B from scratch on the resulting condensed data.
    }
    \label{fig:main_mllm}
    \vspace{-10pt}
\end{figure}

\noindent\textbf{Computational Efficiency Analysis.}
We further analyzed the computational efficiency of different synthesis-based methods in terms of both compression time and GPU memory usage, as summarized in Figure~\ref{fig:efficiency}. Compared to previous baselines such as DC, DM, and MTT, VideoCompressa achieves a remarkable reduction in resource consumption while maintaining superior accuracy. Specifically, \textit{at a 0.52\% compression ratio on UCF101, our method runs over 5800× faster than DC}, requiring only 20.5 seconds for compression, and consumes merely 24.8 GB of GPU memory—saving up to 45 GB compared to existing approaches. Despite its lightweight cost, VideoCompressa surpasses full-data training by 5.43\% in classification accuracy. These results highlight the scalability and efficiency of our approach for large-scale video dataset reconstruction.

\subsection{Results on Multimodal Large Language Models}
The proposed VideoCompressa framework generalizes effectively to MLLMs, demonstrating strong compression capability and transferability across architectures. As illustrated in Figure~\ref{fig:main_mllm}, our method achieves \textit{lossless compression on MLLM}, surpassing full-dataset Supervised Fine-Tuning (SFT) on both HMDB51 and UCF101 while using merely 0.21\% and 0.13\% of the original training data, respectively. This remarkable result indicates that our distilled data preserve both semantic and temporal richness essential for multimodal understanding. Compared to the previous state-of-the-art baseline, \textit{VideoCompressa achieves an 11.4\% absolute improvement on HMDB51 and an 8.2\% gain on UCF101}, establishing a new performance upper bound under extreme data reduction. These results highlight that our approach not only compresses large-scale video datasets efficiently but also enhances generalization for MLLM-based video understanding tasks.

\subsection{Ablation Study}

\begin{figure}[tb!]
    \centering
    \includegraphics[width=0.99\linewidth]{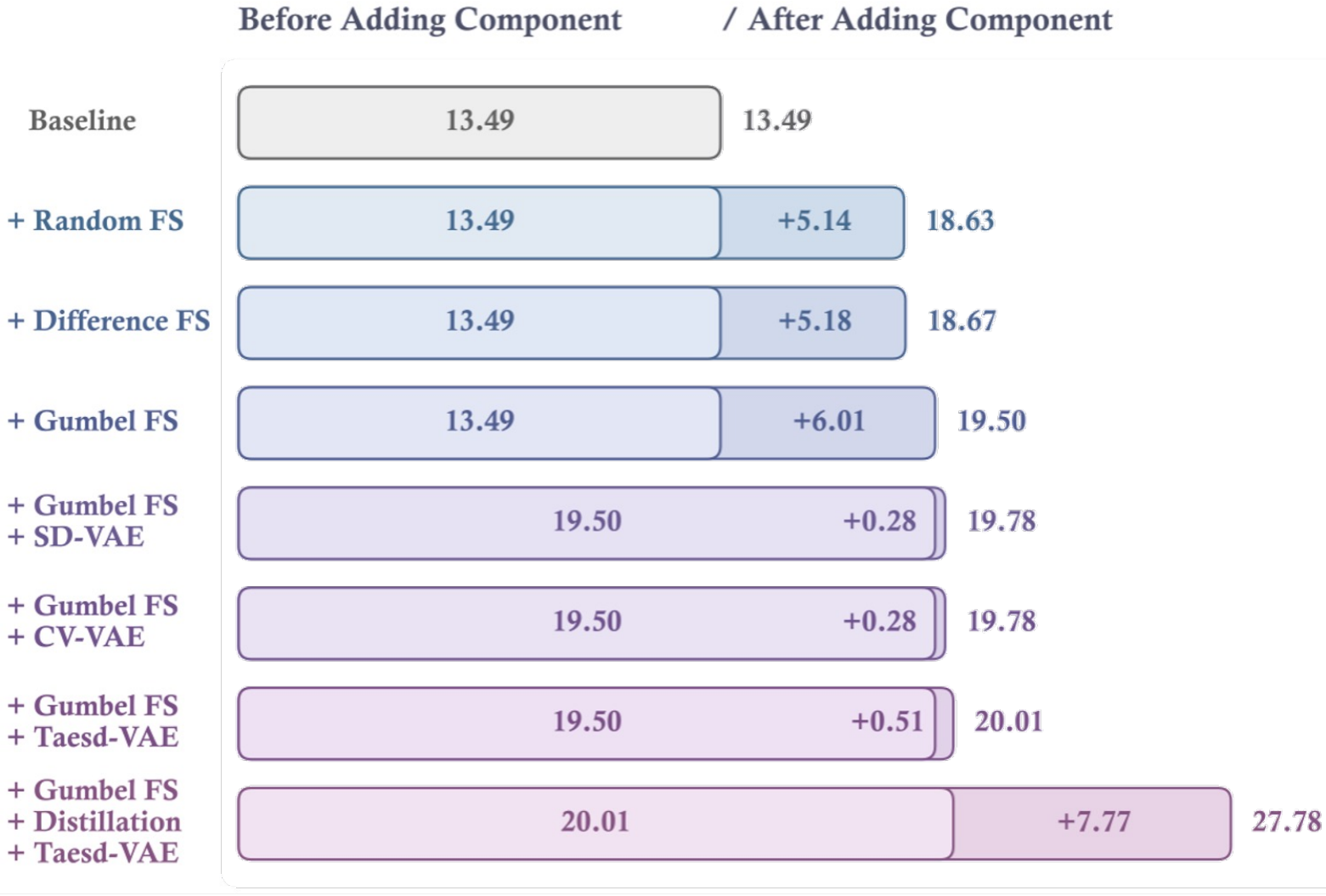} 
    \vspace{-5pt}
    \caption{
Ablation study of individual components. The figure compares the impact of three Frame Selection methods, three VAE variants, and the inclusion of the video understanding module. Experiments are conducted on UCF101 with a 0.13\% data ratio. FS indicates frame selection method.
    }
    \label{fig:ablation}
    \vspace{-5pt}
\end{figure}

We conducted a comprehensive ablation study to meticulously evaluate the contribution of each component within our VideoCompressa framework. The results are illustrated in Figure~\ref{fig:ablation}. Our analysis follows a cumulative design, starting from a random baseline that achieves 13.49\% accuracy.
First, we integrated and compared three different key-frame selection methods. Among them, Gumbel-Max frame selection method provides the most significant improvement, boosting the accuracy by 6.01\% to 19.50\%. Building upon this enhanced model, we then evaluated three VAE variants, where the Taesd-VAE proved most effective, further increasing the accuracy to 20.01\%. Finally, the inclusion of our reconstruction process yields a substantial gain of 7.77\%, bringing the final accuracy to 27.78\%.

\begin{table}[t]
\centering
\caption{
Performance of different key-frame selection strategies on the HMDB51 dataset. To isolate the impact of the selection method itself, the understanding network was disabled, and each strategy selected a fixed budget of 4 frames. Our learnable Gumbel-Max approach significantly outperforms both the no-selection baseline and common heuristic methods.
}
\vspace{-5pt}
\label{tab:ablation_kfs}
\resizebox{0.99\columnwidth}{!}{
\begin{tabular}{c|cccc}
\toprule
Selection Method & None & Uniform & Pixel Diff. & \textbf{Gumbel-Max (Ours)} \\
\midrule
Accuracy 
& 13.49{\scriptsize$\pm$0.08}
& 18.63{\scriptsize$\pm$0.43}
& 18.67{\scriptsize$\pm$0.78}
& \cellcolor{green!10}\textbf{19.50{\scriptsize$\pm$0.82}} \\
\bottomrule
\end{tabular}
}
\vspace{-5pt}
\end{table}

\noindent\textbf{Impact of Key-Frame Selection Strategies.}
We dissect the contribution of the key-frame selection module by conducting an ablation study in a decoupled setting. To strictly isolate the impact of the selection mechanism, we trained the models from scratch using a fixed budget of 4 frames selected by each strategy. We compared our approach against two common heuristics: Uniform Sampling and Pixel Difference-based Selection. As detailed in Table~\ref{tab:ablation_kfs}, the introduction of any structured key-frame selection strategy yields a substantial performance leap over the baseline, confirming that reducing temporal redundancy is critical for data efficiency. Most importantly, our Gumbel-Max strategy achieves the highest accuracy of 19.50\%, surpassing both uniform and difference-based heuristics. It is worth noting that in this setting, the Gumbel selector operates without the supervisory gradients from the training process. The fact that it outperforms handcrafted heuristics even in this decoupled regime validates its intrinsic ability to adaptively identify the most information-rich temporal segments for the task.

\noindent\textbf{Impact of VAE on Computational Efficiency.} 
\begin{table}[tb!]
\centering
\caption{Impact of different VAE architectures on classification accuracy (\%) and computational efficiency on the HMDB51 dataset. To isolate the VAE's contribution, accuracy was evaluated with the synthesis component disabled. Computational efficiency was measured during a single iteration step. For a fair comparison, all configurations select 4 frames using the Gumbel-Max strategy. OOM indicates out-of-memory.}
\vspace{-5pt}
\label{tab:ablation_vae}
\resizebox{0.99\columnwidth}{!}{
\begin{tabular}{l|cccc}
\toprule
VAE Type & No VAE & SD-VAE & CV-VAE & \textbf{Taesd-VAE} \\ 
VAE Size (MB) & - & 335 & 80 & 
\textbf{9} \\
\midrule
Accuracy $\uparrow$
& 19.50
& 19.78
& 19.78
& \cellcolor{green!10}\textbf{20.01} \\
GPU Memory (GB) $\downarrow$& $>$800 & 25.13 & 24.82 & \cellcolor{green!10}\textbf{24.80} \\
Time (s) $\downarrow$& OOM & 388.62 & 57.15 & \cellcolor{green!10}\textbf{22.86} \\
\bottomrule
\end{tabular}
}
\end{table}
To quantify the efficiency gains brought by different VAE designs, we conducted an ablation that replaces the VAE architecture while keeping the rest of the pipeline fixed. As shown in Table~\ref{tab:ablation_vae}, Taesd-VAE significantly improves computational efficiency while preserving accuracy. By compressing videos into a compact latent space, \textit{VAE-based methods reduce the input dimensionality and yield at least $\times$32 reduction in memory consumption}. In particular, \textit{Taesd-VAE achieves $\times$17 speedup over SD-VAE}, making it the most efficient choice without sacrificing performance.

\noindent\textbf{Ablation on coupled effect of the Understanding Network and Gumbel Trick.}
\begin{table}[tb!]
\centering
\caption{Classification accuracy (\%) on HMDB51 under 2 different frame selection methods. For a fair comparison, all experiments were conducted under a fixed
setting that selects 4 frames.}
\vspace{-5pt}
\label{tab:ablation_distill}
\resizebox{0.9\linewidth}{!}{
\begin{tabular}{c|cc}
\toprule
Training Step & Uniform & \textbf{Gumbel Trick (Ours)} \\
\midrule
0 & 18.63{\scriptsize$\pm$0.43} & \cellcolor{green!10}19.50{\scriptsize$\pm$0.82} \\
1 & 18.31{\scriptsize$\pm$0.92} & \cellcolor{green!10}26.27{\scriptsize$\pm$0.38} \\
2 & 18.53{\scriptsize$\pm$1.09} & \cellcolor{green!10}26.47{\scriptsize$\pm$0.46} \\
4 & 19.87{\scriptsize$\pm$0.39} & \cellcolor{green!10}27.78{\scriptsize$\pm$1.14} \\
\bottomrule
\end{tabular}
}
\vspace{-5pt}
\end{table} 
To isolate the contribution of the video understanding network, we evaluated Uniform and Gumbel-Max key-frame selection strategies on HMDB51 under varying numbers of training steps. As shown in Table~\ref{tab:ablation_distill}, the Gumbel Trick frame selector benefits substantially from the presence of the supervisory gradients from the video understanding network, since its parameters are updated through the Gumbel-Softmax. \textit{Notably, Gumbel Trick achieves a 7.91\% accuracy gain at 4 steps}, highlighting the strong coupling between the frame selection network and the understanding component. This further confirms the feasibility of exploiting intra-sample redundancy for video compression and underscores the effectiveness of our proposed framework.

\noindent\textbf{Sensitivity analysis of reconstruction steps.}
\begin{figure}[tb!]
    \centering
    \includegraphics[width=0.99\linewidth]{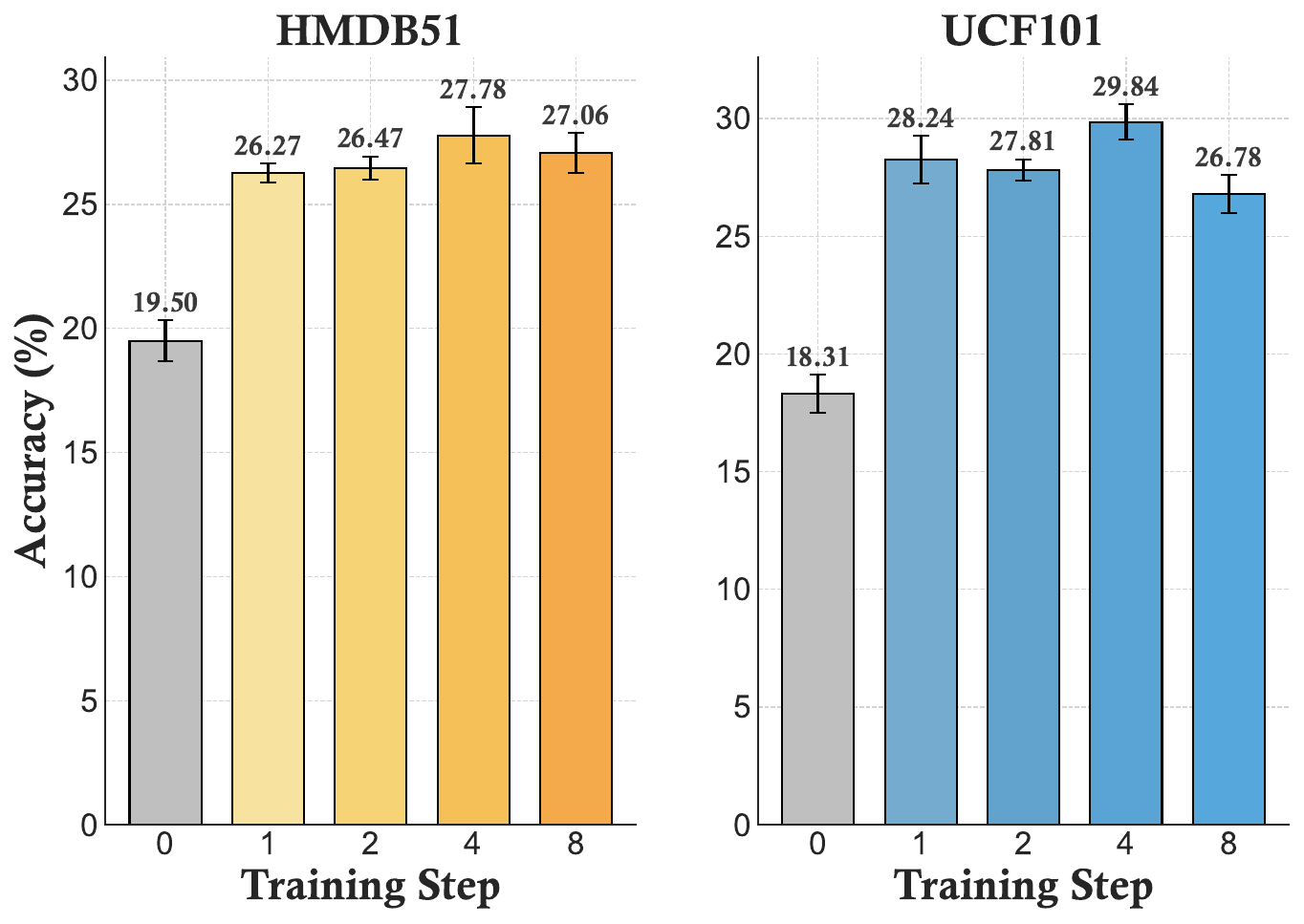}
    \vspace{-5pt}
    \caption{Sensitivity analysis of training steps on HMDB51 and UCF101. All experiments were conducted under a fixed setting: selecting 4 frames with the Gumbel-based strategy.}
    \label{fig:sensitivity}
\vspace{-5pt}
\end{figure}
To assess the robustness of our approach, we performed a sensitivity analysis on both HMDB51 and UCF101. As illustrated in Figure~\ref{fig:sensitivity}, VideoCompressa maintains consistently strong performance across different training steps, demonstrating stable behavior under varying training dynamics.

\section{Discussion}

\subsection{Why Prior Methods Fall Short: Intra-Sample Redundancy Must Be Prioritized Over Inter-Sample Redundancy}

\begin{figure}[tb!]
    \centering
    \includegraphics[width=0.99\linewidth]{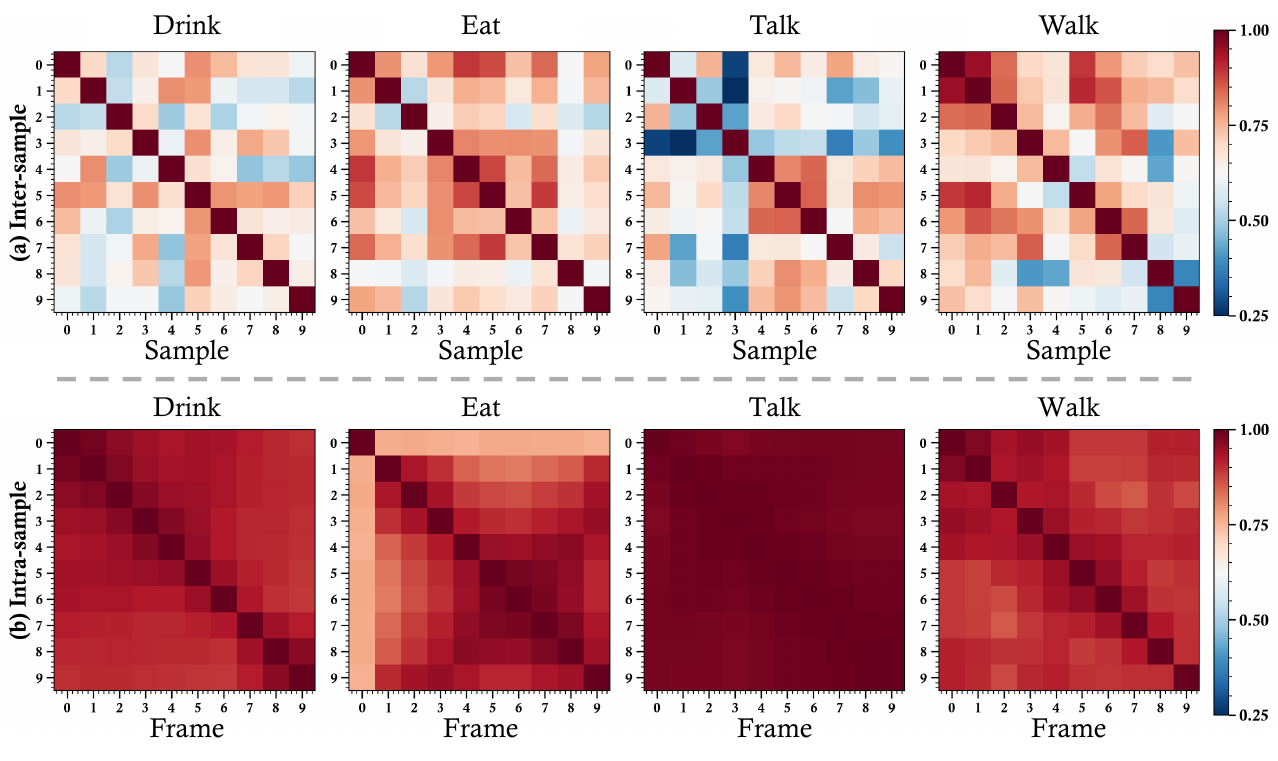}
    \caption{
    Visualization of inter-sample versus intra-sample redundancy for select action classes from the UCF101 dataset. 
    }
    \label{fig:redundancy_analysis}
\vspace{-5pt}
\end{figure}
A critical analysis of prior video dataset compression techniques reveals a fundamental limitation: an overemphasis on resolving \textit{inter-sample} redundancy while largely neglecting the far more significant issue of \textit{intra-sample} redundancy. Many previous approaches, such as coreset selection~\citep{herding,forgetting,grand}, treat each video as an indivisible atomic unit, thereby failing to address the massive data repetition that exists \textit{within} each sample.
This oversight is the principal reason for their suboptimal performance. As shown in Figure~\ref{fig:redundancy_analysis}, The core inefficiency of large-scale video datasets stems not from the similarity between different videos, but from the high temporal correlation between consecutive frames in a single video. The inter-sample correlation matrices show that while videos of the same class are related, there is considerable variance between them, indicated by the mix of high and low correlation values off the main diagonal. Simply discarding entire videos risks losing the diverse expressions of an action.
In stark contrast, the intra-sample matrices demonstrate overwhelmingly high correlation between adjacent frames. The deep red coloring across these heatmaps signifies pervasive redundancy, where consecutive frames often contain nearly identical information. Methods that focus only on inter-sample compression are forced to retain all of this redundant data for any video they select. Consequently, they fail to create a truly dense and efficient data representation. Our work posits that a paradigm shift is necessary: effective video data compression must prioritize the principled selection of informative temporal segments within each video, directly tackling the primary source of data inefficiency.

\subsection{Why Differentiable Frame Selection Works Well: Dynamic Compression Matters}

To elucidate the working principles behind the effectiveness of key-frame selection in video datasets, we conducted a comparative analysis of selection-based and synthesis-based methodologies. 
This investigation was performed on the UCF101 dataset, employing a sampling ratio of 0.13\%. 
We utilized the Uniform Manifold Approximation and Projection (UMAP)~\citep{umap} technique for dimensionality reduction, enabling a 2D visualization of the feature space to assess the data distribution. As illustrated in Figure~\ref{fig:visualization}, the distribution produced by our method, VideoCompressa, most closely mirrors the ground-truth data distribution. This indicates that VideoCompressa achieves a superior trade-off between representativeness and diversity, which we identify as the key principle for its success. For instance, within the \texttt{UnevenBars} category, VideoCompressa's synthesized data points are densely clustered around the distribution's center, effectively capturing the core characteristics of this class. In contrast, the Forgetting method fails to adequately represent this central region. Furthermore, in the \texttt{ApplyEyeMakeup} category, the embedding generated by VideoCompressa demonstrates significant diversity.  The visualization reveals multiple distinct clusters with substantial inter-cluster distances---a topological feature adeptly captured by VideoCompressa but overlooked by the Forgetting method. This observation corroborates previous findings~\citep{Li_2025_ICCV} which suggest that preserving diversity is a critical factor in the performance of dataset synthesis tasks for image-based datasets.

\begin{figure}[tb!]
    \centering
    \includegraphics[width=0.99\linewidth]{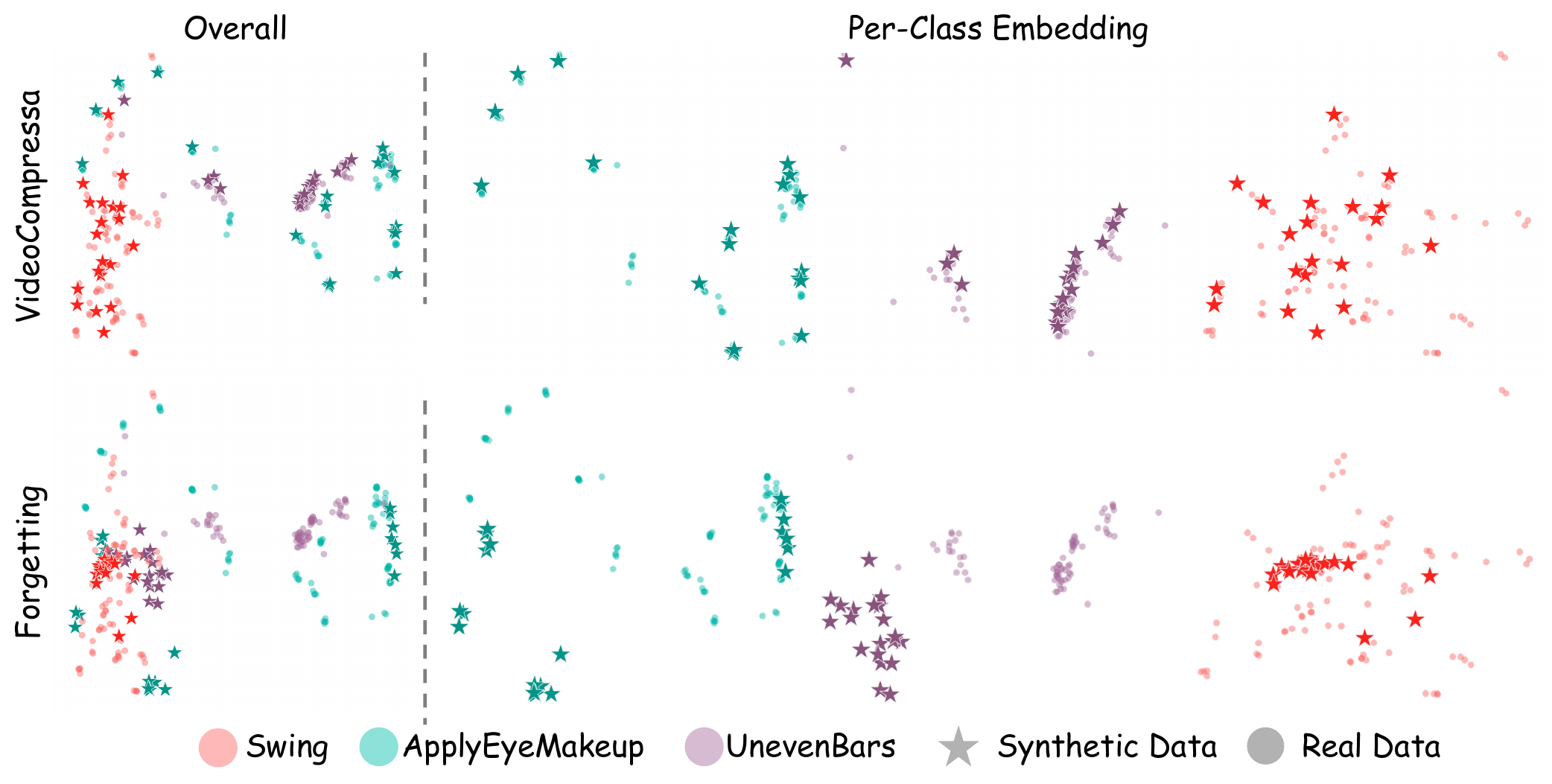}
    \caption{
    UMAP visualization comparing the feature distributions of real data and synthetic data on the UCF101 dataset. }
    \label{fig:visualization}
\vspace{-5pt}
\end{figure}

\section{Qualitative Analysis}
\label{sec:qualitative_analysis}

To rigorously evaluate the effectiveness of VideoCompressa in capturing complex temporal dynamics, we provide a qualitative visualization on the Something-Something V2 (SSv2) dataset. Unlike static image datasets, SSv2 relies heavily on temporal ordering and causal reasoning, making it an ideal testbed for verifying our differentiable keyframe selection mechanism. Figure~\ref{fig:vis_appendix} presents a comprehensive comparison between 16 frames uniformly sampled from the original video, a standard random selection baseline, and the 4-frame sequence reconstructed from our compressed latent codes.

\begin{figure*}[t]
    \centering
    \includegraphics[width=1.0\linewidth]{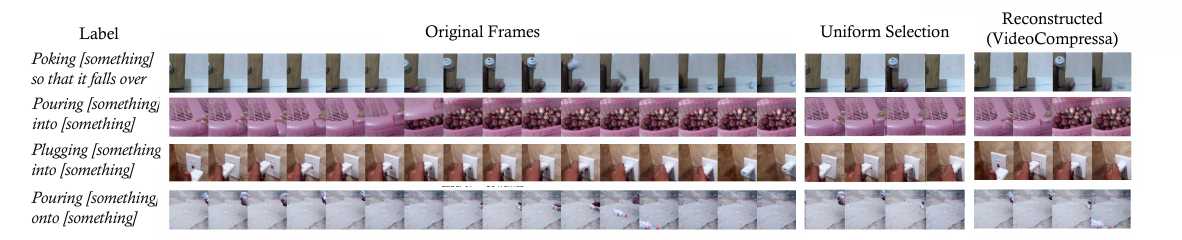}
\caption{Qualitative comparison of temporal compression on Something-Something V2. The figure contrasts 16 uniformly sampled original frames on the left against 4-frame sequences obtained via Uniform Sampling in the middle and our VideoCompressa reconstruction on the right. While Uniform Sampling often yields static snapshots due to rigid striding, our method prioritizes semantically critical moments. For instance, our model captures the causal transition of an object tipping over in the first row and the visual consequence of pouring liquid, specifically the stained ground shown in the fourth row—details entirely missed by the baseline. Additionally, fine-grained interactions in the third row and accumulation dynamics in the second row are faithfully preserved in the reconstructed latent space, verifying that our method retains both temporal logic and spatial semantics.}
    \label{fig:vis_appendix}
\end{figure*}

\noindent\textbf{Superior Temporal Selection.} 
As evidenced by the comparison, Uniform Sampling often misses critical transition states due to its rigid interval selection, as clearly shown in the ``Uniform Selection'' column. In contrast, our Gumbel-based selector demonstrates intelligent budget allocation prioritized by semantics. Specifically, for the action \textit{``Poking [something] so that it falls over''} in the first row, our model explicitly selects the exact moment the object loses balance in the third frame, whereas Uniform Sampling captures only the static pre- and post-action states, effectively severing the causal link of the object's disappearance. Furthermore, in the \textit{``Pouring [something] onto [something]''} example presented in the fourth row, our method retains the frame exhibiting the stained surface—visual evidence of the action's consequence that is essential for distinguishing ``pouring'' from merely ``holding.'' This capability extends to fine-grained interactions as well; as illustrated by the \textit{``Plugging [something] into [something]''} sequence in the third row, our method preserves the precise alignment and insertion phase often lost in strided sampling.

\noindent\textbf{High-Fidelity Spatial Reconstruction.} 
The rightmost column visualizes the decoded output from our optimized latent codes. Despite the significant compression ratio utilizing only about 0.13\% of the original raw data size and the inherent dimensionality reduction of the VAE, the reconstructed frames retain rich semantic details. While minor high-frequency textures may be smoothed, the object identities and motion cues remain distinct. This confirms that VideoCompressa learns to store highly discriminative visual features in the latent space rather than relying on adversarial noise or overfitting to background statistics.

\section{Conclusion}
\label{sec:Conclusion}
In this work, we addressed the critical challenge of prohibitive storage and computational costs associated with large-scale video datasets. We introduced VideoCompressa, a novel dataset synthesis framework that reimagines video compression as a principled temporal information selection problem. Our core innovation lies in the synergy between a differentiable Gumbel keyframe selector and a gradient-guided latent optimization process. This combination enables the model to automatically identify and leverage the most informative frames into a compact, semantically rich latent representation. The success of VideoCompressa opens several exciting avenues for future research. This perspective on ``learning what to ignore'' provides a scalable and interpretable foundation for the next generation of video dataset synthesis. We believe this work will inspire further exploration into dynamic, content-aware data compression, with promising extensions to multimodal learning and the development of large-scale, yet highly efficient, video-language models.


{
    \small
    \bibliographystyle{ieeenat_fullname}
    \bibliography{main}
}
\clearpage
\maketitlesupplementary

\section{Theorm of Gumbel-Based Selection}
\label{sec:appendix_gumbel_proof}

This section provides the theoretical foundation for our differentiable keyframe selection mechanism. We first provide a formal proof for the Gumbel-Max trick (Theorem~\ref{prop:gumbel_max}), demonstrating that adding i.i.d.\ Gumbel noise to logits and taking the index of the maximum element is equivalent to sampling from the categorical distribution. Subsequently, we analyze the gradient approximation (Theorem~\ref{prop:gradient_bound}), showing that the gradients derived from our Gumbel-Softmax relaxation serve as a bounded approximation to the gradients of the discrete Straight-Through Estimator, thereby justifying the consistency of our end-to-end training strategy.

\begin{theorem}[The Gumbel-Max Trick]
\label{prop:gumbel_max}
Let $q_1, \dots, q_K$ be a set of logits. Let $g_1, \dots, g_K$ be i.i.d. random variables from the standard Gumbel distribution. Then:
\[
P\left(j = \arg\max_{i} (q_i + g_i)\right) = \frac{\exp(q_j)}{\sum_{i=1}^K \exp(q_i)}
\]
\end{theorem}

\begin{proof}
The standard Gumbel distribution is defined by its CDF, $F(x) = \exp(-\exp(-x))$, and its PDF, $f(x) = \exp(-x - \exp(-x))$.

Let $k = \arg\max_{i} (q_i + g_i)$. The event $k=j$ occurs if and only if $q_j + g_j$ is greater than all other $q_i + g_i$ for $i \neq j$. We compute this probability by integrating over all possible values of $g_j$. Let $g_j = x$:
\begin{align*}
    P(k=j) &= P(q_j + g_j \ge q_i + g_i, \; \forall i \neq j) \\
           &= \int_{-\infty}^{\infty} P(g_i \le x + q_j - q_i, \; \forall i \neq j) f(x) dx
\end{align*}
Since the $g_i$ are i.i.d., the joint probability inside the integral can be factored into a product of individual CDFs:
\begin{align*}
    P(k=j) &= \int_{-\infty}^{\infty} \left( \prod_{i \neq j} F(x + q_j - q_i) \right) f(x) dx \\
           &= \int_{-\infty}^{\infty} \left( \prod_{i \neq j} \exp(-\exp(-(x + q_j - q_i))) \right) \\
           & \qquad \cdot \exp(-x - \exp(-x)) dx
\end{align*}
The product of exponentials can be rewritten as the exponential of a sum. This is the step where we break the long equation for a two-column layout.
\begin{align*}
    P(k=j) &= \int_{-\infty}^{\infty} \exp\left( -\sum_{i \neq j} \exp(-(x + q_j - q_i)) \right) \\
           & \qquad \cdot \exp(-x) \exp(-\exp(-x)) dx \\
           &= \int_{-\infty}^{\infty} \exp(-x) \exp\left( -\exp(-x) \right. \\
           & \qquad \left. \cdot \left[ 1 + \sum_{i \neq j} \exp(q_i - q_j) \right] \right) dx
\end{align*}
The term in the brackets can be simplified by noting that $1 = \exp(q_j - q_j)$:
\[
    1 + \sum_{i \neq j} \exp(q_i - q_j) = \sum_{i=1}^K \exp(q_i - q_j)
\]
Let $C = \sum_{i=1}^K \exp(q_i - q_j)$. This constant is independent of $x$. The integral simplifies to:
\[
    P(k=j) = \int_{-\infty}^{\infty} \exp(-x) \exp(-C \cdot \exp(-x)) dx
\]
We solve this using the substitution $u = C \cdot \exp(-x)$. This gives $du = -C \cdot \exp(-x) dx$, so $\exp(-x)dx = -\frac{1}{C}du$. The integration limits change from $(-\infty, \infty)$ for $x$ to $(\infty, 0)$ for $u$.
\begin{align*}
    P(k=j) &= \int_{\infty}^{0} \exp(-u) \left(-\frac{1}{C}\right) du \\
           &= \frac{1}{C} \int_{0}^{\infty} \exp(-u) du = \frac{1}{C} [-\exp(-u)]_0^\infty \\
           &= \frac{1}{C} (0 - (-1)) = \frac{1}{C}
\end{align*}
Substituting the definition of $C$ back, we get:
\begin{align*}
    P(k=j) &= \frac{1}{\sum_{i=1}^K \exp(q_i - q_j)} \\
           &= \frac{\exp(q_j)}{\exp(q_j)\sum_{i=1}^K \exp(q_i - q_j)} \\
           &= \frac{\exp(q_j)}{\sum_{i=1}^K \exp(q_i)}
\end{align*}
This is the softmax probability of the $j$-th category, which completes the proof.
\end{proof}

\begin{theorem}[Gradient Approximation Bound]
\label{prop:gradient_bound}
Let $\mathbf{q}$ be the logit vector and $\mathbf{g}$ be the Gumbel noise vector. We define the hard selection vector $\mathbf{y}^h$ via Gumbel-Max and the soft selection vector $\mathbf{y}^s$ (via Gumbel-Softmax with temperature $\tau$) as:
\[
\mathbf{y}^h = \text{one\_hot}\left(\arg\max(\mathbf{q}+\mathbf{g})\right), \quad \mathbf{y}^s = \text{softmax}\left(\frac{\mathbf{q}+\mathbf{g}}{\tau}\right)
\]
Assume the downstream loss function $\mathcal{L}$ is $L$-Lipschitz smooth with respect to the selection weights. The error between the gradient of our method ($\mathbf{g}_{GS}$) and the Straight-Through Estimator gradient ($\mathbf{g}_{ST}$) is bounded by a linear function of $\tau$:
\[
\| \mathbf{g}_{GS} - \mathbf{g}_{ST} \| \le C \cdot \tau
\]
where $C$ is a constant independent of $\tau$.
\end{theorem}

\begin{proof}
The two gradients with respect to the scorer parameters $\phi$ are defined as:
\begin{align}
    \mathbf{g}_{GS} &= \nabla_{\phi} \mathcal{L}(\mathbf{y}^s) = \frac{\partial \mathcal{L}}{\partial \mathbf{y}^s} \frac{\partial \mathbf{y}^s}{\partial \mathbf{q}} \frac{\partial \mathbf{q}}{\partial \phi} \\
    \mathbf{g}_{ST} &= \nabla_{\phi}^{STE} \mathcal{L} \approx \frac{\partial \mathcal{L}}{\partial \mathbf{y}^h} \frac{\partial \mathbf{y}^s}{\partial \mathbf{q}} \frac{\partial \mathbf{q}}{\partial \phi}
\end{align}
Note that STE uses the hard sample $\mathbf{y}^h$ for the forward pass (loss computation) but approximates the backward pass using the soft Jacobian $\frac{\partial \mathbf{y}^s}{\partial \mathbf{q}}$.
The difference in gradients is:
\[
    \| \mathbf{g}_{GS} - \mathbf{g}_{ST} \| = \left\| \left( \frac{\partial \mathcal{L}}{\partial \mathbf{y}^s} - \frac{\partial \mathcal{L}}{\partial \mathbf{y}^h} \right) \frac{\partial \mathbf{y}^s}{\partial \mathbf{q}} \frac{\partial \mathbf{q}}{\partial \phi} \right\|
\]
Using the Cauchy-Schwarz inequality and the sub-multiplicativity of norms:
\[
    \| \mathbf{g}_{GS} - \mathbf{g}_{ST} \| \le \left\| \frac{\partial \mathcal{L}}{\partial \mathbf{y}^s} - \frac{\partial \mathcal{L}}{\partial \mathbf{y}^h} \right\| \cdot \left\| \frac{\partial \mathbf{y}^s}{\partial \mathbf{q}} \frac{\partial \mathbf{q}}{\partial \phi} \right\|
\]
By the $L$-Lipschitz smoothness assumption of $\mathcal{L}$, the first term is bounded by the distance between the soft and hard samples:
\[
    \left\| \frac{\partial \mathcal{L}}{\partial \mathbf{y}^s} - \frac{\partial \mathcal{L}}{\partial \mathbf{y}^h} \right\| \le L \| \mathbf{y}^s - \mathbf{y}^h \|
\]
It is a known property of the Softmax function that $\mathbf{y}^s(\tau)$ converges to $\mathbf{y}^h$ as $\tau \to 0$. Specifically, for a fixed set of distinct logits, the approximation error is bounded by $\mathcal{O}(\tau)$. Thus, there exists a constant $C_1$ such that $\| \mathbf{y}^s - \mathbf{y}^h \| \le C_1 \tau$.

Let $C_2 = \sup \| \frac{\partial \mathbf{y}^s}{\partial \mathbf{q}} \frac{\partial \mathbf{q}}{\partial \phi} \|$ be the upper bound of the gradient magnitude, which is finite in practice due to bounded activations and weights.
Combining these inequalities:
\[
    \| \mathbf{g}_{GS} - \mathbf{g}_{ST} \| \le (L \cdot C_1 \cdot \tau) \cdot C_2 = (L C_1 C_2) \cdot \tau
\]
Setting $C = L C_1 C_2$, we obtain $\| \mathbf{g}_{GS} - \mathbf{g}_{ST} \| \le C \cdot \tau$.
\end{proof}

\section{Pseudo Code of VideoCompressa}
Our method, VideoCompressa, reframes video data synthesis as a dynamic latent compression problem for data-efficient video understanding. As shown in the Algorithm~\ref{alg:ours}, the algorithm consists of three stages. In the Latent Code Initialization stage, each video frame in the real dataset is encoded into a compact latent representation using a pretrained and frozen VAE, shifting optimization from pixel space to a structured latent space. In the Joint Optimization stage, we jointly train the latent codes, a frame scorer, and a video understanding network. Latent codes are decoded, scored, and softly sampled via the Gumbel-Softmax trick, which enables gradients from the classification loss to update all components end-to-end, including the selected latent codes. In the Deterministic Keyframe Extraction stage, the trained scorer selects the top-$K$ most informative frames for each video. The corresponding optimized latent codes form the final compressed synthetic dataset, which is significantly smaller than the original data but highly informative for downstream tasks.

\begin{algorithm}[t]
\caption{VideoCompressa: Joint Temporal Compression and Spatial Reconstruction}
\label{alg:ours}
\begin{algorithmic}[0]
\State \textbf{Input:} Real dataset $\mathcal{D}_{\text{train}} = \{(V_i, y_i)\}_{i=1}^N$, pretrained VAE $(\mathcal{E},\psi)$, number of keyframes $K$, random initialized Frame scorer $\phi$, Understanding network $\mathcal{M}_\theta$.
\State \textbf{Learnable Variables:} Latent codes $\{Z_i\}$, scorer parameters $\phi$, Understanding network parameters $\theta$.
\Statex \textcolor{magenta}{\textbf{// Stage 1: Latent Code Initialization}}
\For{each video $V_i = \{f_{i,t}\}_{t=1}^{T_i}$ in $\mathcal{D}_{\text{train}}$}

    \State $Z_i \gets \{\mathcal{E}(f_{i,t})\}$
\EndFor

\Statex \textcolor{magenta}{\textbf{// Stage 2: Gumbel-Based Temporal Compression}}
\While{not converged}
    \State Sample mini-batch $\mathcal{B}$.
    \For{each $i \in \mathcal{B}$}
        \State $\hat{V}_i = \{\psi(z_{i,t})\}_{t=1}^{T_i}$ 
        \State $q_i = \phi(\hat{V}_i)$ 
        \For{$k = 1, \dots, K$}
            \State Sample Gumbel noise $g_i^{(k)}$
            \State $p_i^{(k)} \gets \mathrm{Softmax}\big((q_i + g_i^{(k)}) / \tau\big)$
            \State $\tilde f_i^{(k)} \gets \sum_{t=1}^{T_i} p_{i,t}^{(k)} \hat f_{i,t}$
        \EndFor
        \State $\tilde{V}_i \gets \{\tilde f_i^{(k)}\}_{k=1}^K$
        \State $\hat y_i = \mathcal{M}_\theta(\tilde{V}_i)$
    \EndFor
    \State Compute loss $\mathcal{L} = \frac{1}{|\mathcal{B}|}\sum_{i \in \mathcal{B}} \mathrm{CE}(\hat y_i, y_i)$
    \State Update $\{Z_i\}$, $\phi$, $\theta$ using gradients from $\mathcal{L}$
\EndWhile

\Statex \textcolor{magenta}{\textbf{// Stage 3: Latent Spatial Reconstruction}}
\For{each video $i$}
    \State $\hat{V}_i = \{\psi(z_{i,t})\}_{t=1}^{T_i}$
    \State $q_i = \phi(\hat{V}_i)$
    \State $T_i^K \gets \mathrm{top\mbox{-}K}(q_i)$
    \State $Z_i^* \gets \{z_{i,t}\}_{t \in T_i^K}$ 
\EndFor
\State \textbf{Output:} Synthetic latent dataset $\mathcal{S}^* = \{Z_i^*\}_{i=1}^N$.

\end{algorithmic}
\end{algorithm}

\section{Implementation Details}

\noindent\textbf{Latent Code Initialization.} 
To initialize the latent representations, we employed a pre-trained Taesd-VAE encoder. All video frames were processed in batches of 8. To ensure a compact latent space, the VAE was configured with channel pruning on 30\% of its parameters and Tucker/PCA compression at a ratio of 0.75, which further reduced the dimensionality of the initial latent codes.

\noindent\textbf{Frame Scorer and Gumbel-Softmax Parameters.} 
The frame scorer is a lightweight convolutional network designed to produce a scalar importance logit for each frame. In our experiments, this module consisted of two $3\times3$ convolutional layers followed by a linear projection layer. For end-to-end training, we utilized the Gumbel-Softmax technique to enable differentiable sampling of keyframes. The Gumbel-Softmax distribution was controlled by a temperature parameter set to $\tau=1.0$. The frame scorer network was optimized using an Adam optimizer with a learning rate of $1\times10^{-3}$.

\noindent\textbf{Synthetic Latent Code Optimization.}
The synthetic latent codes were optimized jointly with the frame scorer and the downstream video understanding network, guided solely by the cross-entropy classification loss. During each training iteration, the latent codes were updated for 4 optimization steps using a batch size of 64. We used an Adam optimizer with a learning rate of $1\times10^{-2}$ for this process.

\noindent\textbf{Evaluation of the synthesized dataset.} 
To assess the quality of the synthesized dataset, we trained a new instance of the video understanding network from scratch using the distilled latent representations. This network was trained for 100 epochs using Stochastic Gradient Descent (SGD). The training hyperparameters were set as follows: a learning rate of 0.01, a batch size of 256, a momentum of 0.9, and a weight decay of $5\times10^{-4}$.

\section{More Experiments of Sensitivity Analysis}
\begin{table}[t]
\vspace{-5pt}
\centering
\caption{
Sensitivity analysis of the number of distillation steps on the HMDB51 and UCF101 datasets. The results demonstrate that our method converges rapidly, achieving a significant performance gain after the first step and peaking after four steps. For this analysis, all configurations use the Gumbel-Max key-frame selector with a 4-frame budget.
}
\vspace{-5pt}
\label{tab:ablation_hmdb_step}
\resizebox{0.99\linewidth}{!}{
\begin{tabular}{c|ccccc}
\toprule
\textbf{Training step} & 0 & 1 & 2 & \textbf{4} & 8 \\ 
\midrule
HMDB51 & 19.50 ± 0.82 & 26.27 ± 0.38 & 26.47 ± 0.46 & \textbf{27.78 ± 1.14} & 27.06 ± 0.81 \\
UCF101 & 18.31 ± 0.82 & 28.24 ± 1.01 & 27.81 ± 0.46 & \textbf{29.84 ± 0.75} & 26.78 ± 0.81 \\
\bottomrule
\end{tabular}
}
\vspace{-5pt}
\end{table}
To investigate the impact of the number of distillation steps on our method's performance, we conducted a sensitivity analysis on the HMDB51 and UCF101 datasets. Table~\ref{tab:ablation_hmdb_step} presents the classification accuracy as a function of the number of optimization steps, where all experiments utilize the Gumbel-Max key-frame selector with a 4-frame budget.

The results clearly demonstrate that our framework converges remarkably fast. A significant performance gain is achieved after just the first distillation step for both datasets. The accuracy continues to improve and peaks at four steps, reaching 27.78\% on HMDB51 and 29.84\% on UCF101. Interestingly, extending the process to eight steps leads to a slight degradation in performance, suggesting that four steps are sufficient to achieve optimal results while avoiding potential overfitting to the guiding network. This rapid convergence underscores the efficiency of our joint optimization strategy in learning a compact and effective data representation.

\end{document}